\documentclass{article}

\PassOptionsToPackage{numbers, compress}{natbib}
\usepackage[final]{nips_2017}

\usepackage[utf8]{inputenc} 
\usepackage[T1]{fontenc}    
\usepackage{hyperref}       
\usepackage{url}            
\usepackage{booktabs}       
\usepackage{amsfonts}       
\usepackage{nicefrac}       
\usepackage{microtype}      
\usepackage{amsmath}
\usepackage{amsthm}
\usepackage{amssymb}
\usepackage{url}
\usepackage[noend]{algorithmic}
\usepackage{algorithm}
\usepackage{graphicx}
\usepackage{caption}
\usepackage{subcaption}
\usepackage{color}
\usepackage{natbib}
\usepackage{wrapfig}

\def\E{\mathbb{E}}

\DeclareMathOperator*{\argmin}{argmin}

\newtheorem{theorem}{Theorem}
\newtheorem{definition}{Definition} 
\newtheorem{assumption}{Assumption} 
\newtheorem{example}{Example}
\newtheorem{proposition}{Proposition}

\newtheorem{remark}{Remark}
\title{\Large{Medoids in almost linear time via multi-armed bandits}}
\author{Vivek Bagaria$^{*}$
\\ Stanford University
\\ vbagaria@stanford.edu
\And Govinda M. Kamath$^*$ 
\\ Stanford University
\\ gkamath@stanford.edu
\And Vasilis Ntranos$^*$ 
  \\ UC, Berkeley
  \\ ntranos@berkeley.edu
\And Martin J. Zhang\thanks{Contributed
 equally and listed alphabetically.}
 \\ Stanford University
 \\ jinye@stanford.edu
  \And David Tse 
  \\ Stanford University
  \\ dntse@stanford.edu}
\begin{document}

\maketitle

\begin{abstract}
Computing the medoid of a large number of points in high-dimensional space is an increasingly common operation in many data science problems. We present an algorithm \texttt{Med-dit} which uses $O(n\log n)$ distance evaluations to compute the medoid with high probability. \texttt{Med-dit} is based on a connection with the multi-armed bandit problem. We evaluate the performance of \texttt{Med-dit} empirically on the Netflix-prize and the single-cell RNA-Seq datasets, containing hundreds of thousands of points living in tens of thousands of dimensions, and observe a $5$-$10$x improvement in performance over the current state of the art. \texttt{Med-dit} is available at \url{https://github.com/bagavi/Meddit}.
\end{abstract}

\vspace{-0.15 in}

\section{INTRODUCTION}

An important building block in many modern data analysis problems such as clustering is the efficient computation of a representative point for a large set of points in high dimension. A commonly used representative point is the {\em centroid} of the set, the point which has the smallest average distance from 
 all other points in the set. For example, k-means clustering \citep{Ste56,macqueen1967some,Llo82} computes the centroid of each cluster under the squared Euclidean distance. In this case, the centroid is the arithmetic mean of the points in the cluster, 
which can be computed very efficiently. Efficient computation of  centroid is central to the success of k-means, as this computation has to be repeated many times in the clustering process.
 
 While commonly used, centroid suffers from several drawbacks. First, the centroid is in general {\em not} a point
in the dataset and thus may not be
interpretable in many applications. This is especially true when
the data is structured like in images, or 
when it is sparse like in recommendation
systems \citep{LesRajUll}. Second, the centroid is sensitive to outliers: several far away points will significantly affect the location of the centroid. 
Third, while the centroid can be efficiently computed under squared Euclidean distance, there are many applications
where using this distance measure is not suitable. Some examples would be 
applications where the data
is categorical like  
medical records \citep{UKbiobank};
or situations
where the data points 
have different
support sizes such as in
recommendation
systems \citep{LesRajUll};
or cases where the 
data points are on
a high-dimensional probability
simplex like in single cell
RNA-Seq analysis 
\citep{NtrKamZhaPacTse}; or
cases where the data lives in a 
space with no well known Euclidean
space like while clustering on 
graphs from social networks.    

An alternative to the centroid is the {\em medoid}; this is the point  {\em in the set} that minimizes the average distance to all the other points. It is used for example in  $k$-medoids clustering \citep{kaufman1987clustering}. 
On the real line, the medoid is the median of the set of points. The medoid overcomes the first two drawbacks of the centroid: the medoid is by definition one of the points in the dataset, and it is less sensitive to outliers than the centroid. 
In addition, centroid algorithms are usually specific to the distance used to define the centroid. 
On the other hand, medoid algorithms usually work for arbitrary distances.

The naive method to compute the medoid
would require computing all pairwise
distances between points in the 
set.
For a set with  $n$ points, 
this would 
require the computation of $\binom{n}{2}$ distances,
which would be computationally prohibitive
when there are hundreds of thousands of points
and each point lives
in a space with dimensions in tens
of thousands.

In the one-dimensional case, the medoid problem 
reduces to the problem of finding the median, 
which can be solved in linear 
time through \texttt{Quick-select} \citep{Hoa61}. 
However, in higher dimensions,
no linear-time algorithm is known.
 \texttt{RAND} \citep{EppWan} 
is an algorithm that estimates the average distance of each point to all the other points by sampling a random subset of other points. It takes a total of  
$O\left(\frac{n \log n}{\epsilon}\right)$ distance 
computations
to approximate the medoid within a 
factor of $(1+\epsilon\Delta)$ with high probability,
where $\Delta$ is the maximum
distance between two points
in the dataset. 
We remark that this is an approximation algorithm, and moreover
$\Delta$ may \textit{not} be known apriori.
\texttt{RAND} was leveraged by 
\texttt{TOPRANK} \citep{OkaCheLi} which 
uses the estimates obtained by \texttt{RAND} to focus on a small subset of candidate points, evaluates the average distance of these points {\em exactly}, and picks the minimum of those. \texttt{TOPRANK} needs  
$O(n^{\frac{5}{3}} \log^{\frac{4}{3}} n)$ distance computations 
to find the {\em exact} medoid with high probability 
under a distributional assumption 
on the average distances.
\texttt{trimed} \citep{NewFle} 
presents a clever algorithm 
to find the medoid with 
$O(n^{\frac{3}{2}} 2^{\Theta(d)})$
distance evaluations,
which works very well
when 
the points  live in
a space of dimensions less
than $20$
(i.e., $2^d \ll \sqrt{n}$) 
under a distributional 
assumption on the points
(in particular on the distribution of points around the medoid).
However, the 
exponential dependence
on dimension 
makes it impractical
when $d \ge 50$,
which is the case
we consider here. 
Note that their result requires the distance measure to satisfy the triangle inequality.


\begin{figure}[t]
\centering
\includegraphics[width=0.6\textwidth]{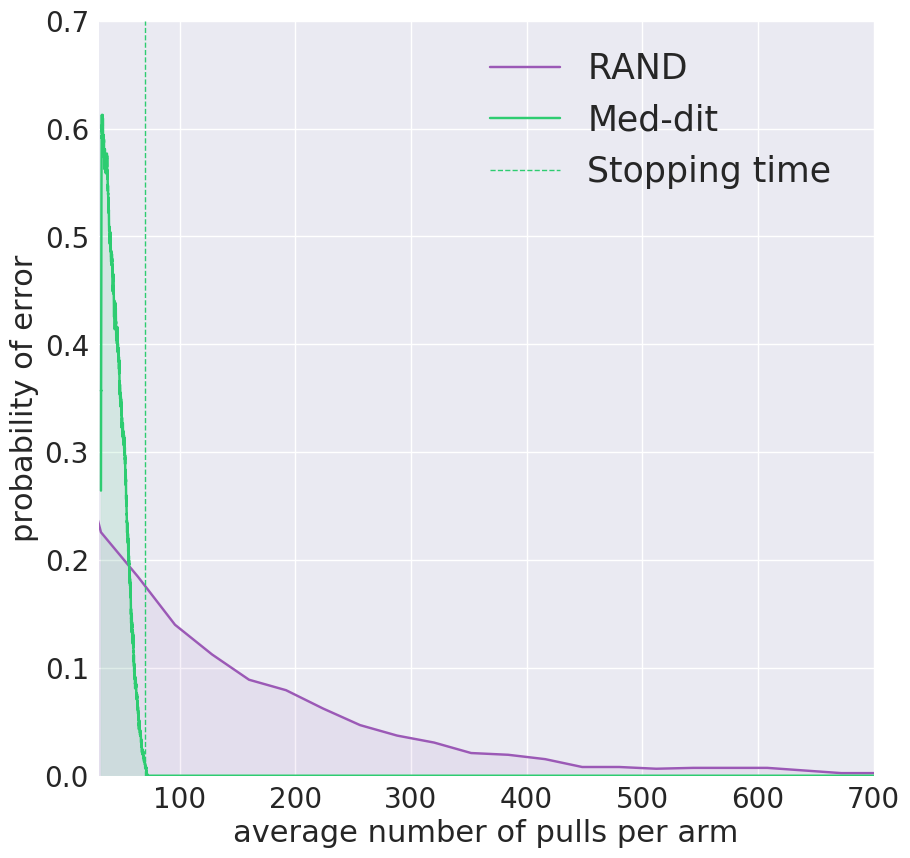}
\caption{\label{fig:performance_RNA_100k} We
 plot the probability that \texttt{RAND}
and \texttt{Med-dit} do not return 
the true medoid as a function 
of the number of distance 
evaluations per point over
$1000$ monte-carlo trials. 
We note that \texttt{Med-dit}
needs much fewer distance
evaluations than  \texttt{RAND}
for comparable performance.
The
computation of the true medoid here 
is also computationally prohibitive 
and is discussed in Section \ref{sec:empi}.
} 
\end{figure}

In this paper, 
we present \texttt{Med(oid)-(Ban)dit}, 
a sampling-based algorithm 
that finds the 
\textit{exact} medoid with high
probability.
It takes $O(n \log n)$ distance 
evaluations under natural assumptions on the distances, which are justified by 
the characteristics observed in real data (Section \ref{sec:prob}). In contrast to \texttt{trimed}, the number of distance computations is {\em independent} of the dimension $d$;  moreover, there is no specific requirement for the distance measure, e.g. the triangle inequality or the symmetry. Thus, \texttt{Med-dit} is particularly suitable for high-dimensional data and general distance measures. 

In
Figure \ref{fig:performance_RNA_100k}, 
we 
showcase the performance of \texttt{RAND}
and \texttt{Med-dit} 
on 
the largest cluster of
the single cell RNA-Seq gene expression
dataset of \citep{10xdata}.
This consists of $27,998$
gene-expressions 
of each of $109,140$ cells, $i.e$
$109,140$ points in $27,998$ dimensions.
We note that 
\texttt{Med-dit}
evaluates around
$10$ times fewer 
distances than 
\texttt{RAND}
to achieve similar performance.
At this scale running 
\texttt{TOPRANK} and
\texttt{trimed}  are
computationally prohibitive.

The main idea behind 
\texttt{Med-dit} is noting that the problem 
of computing the medoid can be posed
as that of computing the best arm
in a multi-armed bandit (MAB) setting 
\citep{even2002pac, jamieson2014best, LaiRob}. 
One views each point in the medoid problem
as an arm whose unknown parameter is its average distance to all the other points. Pulling an arm corresponds 
to evaluating the distance of that point to a
randomly chosen point, which provides an estimate of the arm's unknown parameter. 
We leverage the
extensive literature on the multi-armed bandit problem
to propose \texttt{Med-dit}, which is 
a variant of the Upper Confidence Bound (UCB)
Algorithm \citep{LaiRob}.

Like the other sampling based algorithms  \texttt{RAND} and \texttt{TOPRANK},
\texttt{Med-dit} also aims to estimate
the average distance of every point to the other points by
evaluating its distance to a
random subset of points. However, unlike these algorithms where each point receives a fixed amount of sampling decided {\em apriori}, the sampling in \texttt{Med-dit} is done {\em adaptively}. 
More specifically,
\texttt{Med-dit} maintains 
confidence intervals
for the average distances of all points
and adaptively 
evaluates distances  of only
those points which could 
\textit{potentially} be the medoid.
Loosely speaking, 
points whose lower confidence
bound on the average distance are small
form the set of points that 
could \textit{potentially} be the medoid.
As the algorithm proceeds,
additional 
distance evaluations narrow the 
confidence intervals.
This rapidly
shrinks the set of 
points that
could \textit{potentially} be the medoid,
enabling the algorithm to 
declare a medoid 
while
evaluating
a few distances.

In Section \ref{sec:prob} we delve more into
the problem formulation and assumptions.
In Section \ref{sec:alg}
we present \texttt{Med-dit} and 
analyze
its performance. We extensively 
validate the  
performance of \texttt{Med-dit}
empirically
on two large-scale datasets: a 
single cell RNA-Seq gene expression dataset 
of \citep{10xdata}, and the Netflix-prize
dataset of \citep{netflixprize} in Section \ref{sec:empi}.

\vspace{-0.1 in}

\section{PROBLEM FORMULATION}\label{sec:prob}
Consider $n$ points $x_{1},x_{2},\cdots,x_{n}$ lying in
 some space $\mathcal{U}$ equipped with the distance 
 function
  $d: \mathcal{U}\times \mathcal{U} \mapsto \mathbb{R}_+$. 
Note that we do not assume the triangle inequality or the symmetry for the distance function $d$. 
Therefore the analysis here encompasses directed graphs, or distances
like Bregman divergence
and squared Euclidean distance.
We use the standard notation of $[n]$ 
to refer to the set $\{1,2,\cdots,n\}$.
Let $d_{i,j}\triangleq d(x_i,x_j)$ and let the average distance of a point $x_i$ be 
\begin{align*}
\mu_i \triangleq \frac{1}{n-1}\sum_{j \in [n] - \{ i \}} d_{i,j}.
\end{align*}
The medoid problem 
can be defined as follows.
\begin{definition} (The medoid problem)
For a set of points $\mathcal{X} = \{x_{1},x_{2},\cdots,x_{n} \}$, the medoid is the  
point in $\mathcal{X}$ that has the smallest average distance to other points.
Let $x_{i^*}$ be the medoid. 
The index $i^*$ and the average distance $\mu^*$ of the emdoid are given by 
\begin{align*}
i^{*} \triangleq \argmin_{i\in[n]} \mu_i,~~~~~~~~
\mu^* \triangleq \min_{i\in[n]} \mu_i.
\end{align*}
The problem of finding the medoid of $\mathcal{X}$ is called its
medoid problem.
\end{definition} 
 
In the worst case
over distances
and points,
we would need to 
evaluate $O(n^2)$ distances to compute 
the medoid. Consider the following example.

\begin{example} \label{exp:worst_case} 
Consider a setting where there are $n$ points.
We pick a point $i$ uniformly at random from $[n]$
and set all distances to point $i$ to $0$. 
For every other point $j \in [n] - \{i\}$, we pick a point 
$k$ uniformly at random from $  [n] - \{i, j\}$ and
set $d_{j,k} = d_{k,j} = n$. 
As a result, each point apart from $i$ has a distance of $n$
to at least one point. 
Thus picking the point with distance
$0$ would require at least $O(n^2)$
distance evaluations.
We note that this distance does not satisfy the
triangle inequality, and thus is not a metric.
\end{example}

The issue with Example \ref{exp:worst_case} is that, 
the empirical distributions 
of distances of any point (apart from the true medoid) 
have heavy tails. 
However, such worse-case instances rarely arise in practice. 
Consider the following real-world example.

\begin{example} \label{exp:real_world_data}
We consider the $\ell_1$
distance between each pair 
of points in a cluster -- $109,140$ points in $27,998$
 dimensional probability simplex
 -- taken from a single cell RNA-Seq expression 
dataset 
from  \citep{10xdata}. 
Empirically the $\ell_1$ distance
for a given point follows a 
Gaussian-like distribution 
with a small variance,
as shown 
in  Figure \ref{fig:l1_dist_100k}. 
Modelling  distances among
high-dimensional points
by a Gaussian distribution is
also discussed in \citep{LinSte}. 
Clearly the distances 
do not follow a heavy-tailed
distribution like Example \ref{exp:worst_case}.
\end{example}

\begin{figure}[t]
\centering
\includegraphics[width=\linewidth]{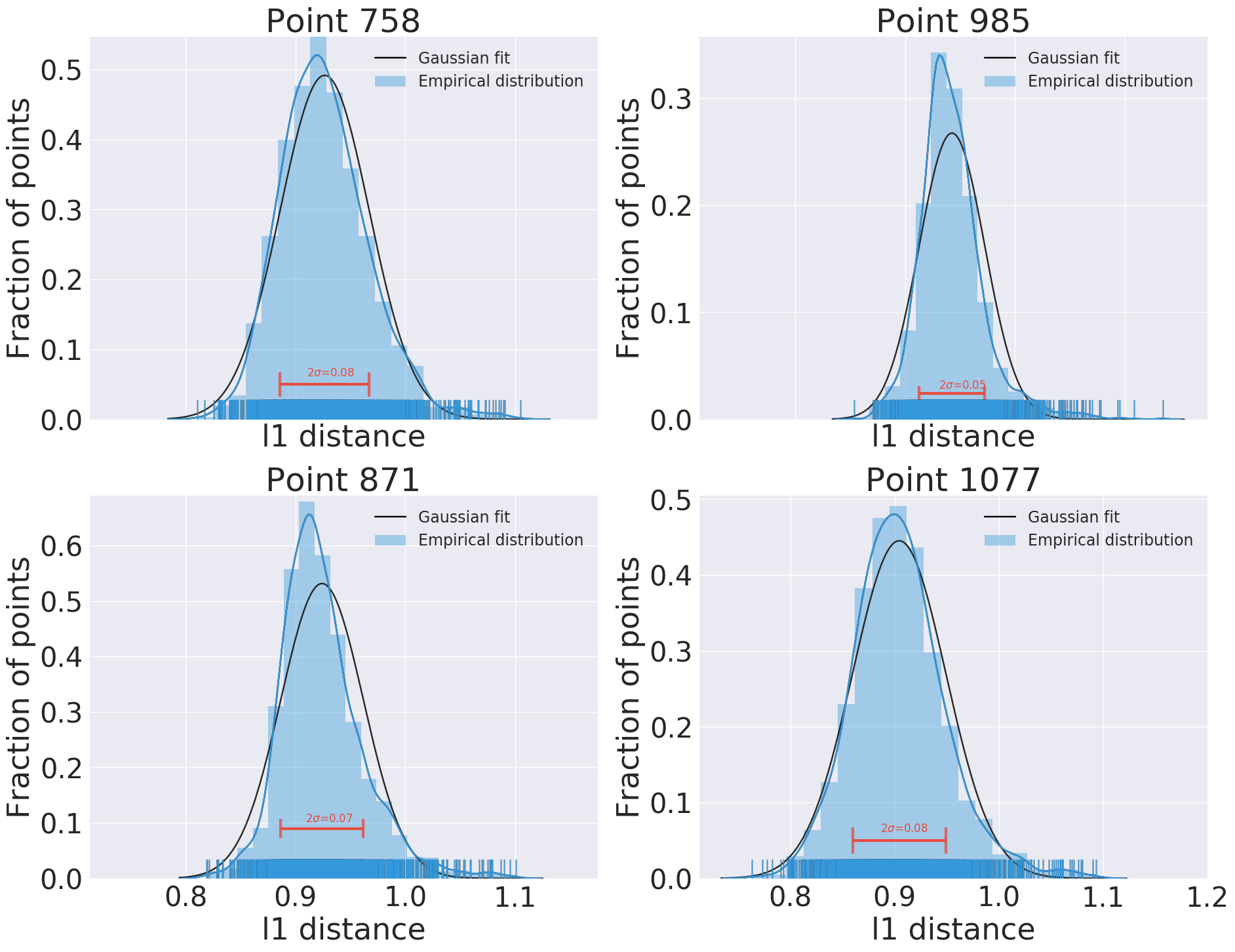}
\caption{\label{fig:l1_dist_100k} Histogram of  
$\ell_1$ distance of $4$ randomly chosen points
 (from $20,000$ points). We note that the standard deviations are around $0.05$.} 
\end{figure}

\subsection{The Distance Sampling Model}
Example \ref{exp:real_world_data} suggests that for any point $x_i$, we can estimate its average distance $\mu_i$ by the empirical mean of its distance from a subset of randomly sampled points. 
Let us formally state this. 
For any point $x_i$, let $\mathcal{D}_i=\{d_{i,j},j\neq i\}$ be the set of distances associated with $x_i$.
Define the random distance samples $D_{i,(1)},D_{i,(2)},\cdots,D_{i,(k)}$ to be i.i.d. \textit{sampled with replacement} from $\mathcal{D}_i$.
We note that 
\begin{align*}
\E[D_{i,(j)}] = \mu_i,~\forall i\in [n],j\in [k]
\end{align*}
Hence the average distance $\mu_i$ can be estimated as 
\begin{align*}
\hat{\mu}_i = \frac{1}{k} \sum_{i=1}^k D_{i,(k)}.
\end{align*} 
To capture the concentration property illustrated in Example \ref{exp:real_world_data}, we assume the random variables $D_{i,(j)}$, i.e. sampling with replacement from $\mathcal{D}_i$, are $\sigma$-sub-Gaussian
\footnote{$X$ is a $\sigma$-sub-Gaussian if  $P(X > t) \le 2 e^{-\frac{t^2}{\sigma^2}}$.  }
 for all $i$.

\subsection{Connection with Best-Arm Problem }
For points whose average distance is close to that of the medoid, we need an accurate 
estimate of their
average distance, and for points  whose average distance 
is far away, coarse estimates suffice. 
Thus the problem reduces 
to  adaptively choosing the number of distance evaluations for 
each
point  to ensure that 
on one hand, we have a good enough estimate
of the average distance to compute the true medoid, while
on the other hand,
we minimise the total number of distance
evaluations.

This problem has been addressed 
as the best-arm identification problem
in the multi-armed bandit (MAB) 
literature (see the  review 
article of \citep{jamieson2014best} 
for instance).
In a typical setting, we have $n$ arms.
At the time $t=0,1,\cdots$, we decide 
to pull an arm $A_t \in [n]$, and receive 
a reward $R_t$ with $\E[R_t]=\mu_{A_t}$.
The goal is to identify the arm with the
largest expected reward with high 
probability while
pulling as few arms as possible.

The medoid problem can be formulated as a best-arm problem as follows:
let each point in the ensemble be an arm and let the average distance $\mu_i$ be the
loss of the $i$-th arm.
Medoid corresponds to the arm with the
 \textit{smallest} loss, i.e the best arm.
At each iteration $t$, pulling arm $i$ is equivalent to sampling (with replacement) from $\mathcal{D}_i$ with expected value $\mu_i$. 
As stated previosuly, the goal is to find the best arm (medoid) with as few pulls (distance computations) as possible.

\vspace{-0.1 in}

\section{ALGORITHM}\label{sec:alg}

We have $n$ points $x_1,\cdots,x_n$.
Recall that the average 
distance for arm $i$ is  
$\mu_i= \frac{1}{n-1} \sum_{j\neq i} d_{i,j}$.
At iteration $t$ of the algorithm, we 
evaluate a distance $D_t$ to point
 $A_t \in [n]$. 
Let $T_i(t)$ be the number of 
distances of 
point $i$ evaluated 
upto time $t$. 
We use a variant of the 
Upper Confidence Bound (UCB) 
\citep{LaiRob} algorithm 
to sample distance of a point 
$i$ with another point chosen
uniformly at random.

We compute the 
empirical mean and use it 
as an estimate
of  $\mu_i$ at time $t$,
in the initial stages of the algorithm. 
More concretely when $T_i(t) < n$,
\begin{align*}
\hat{\mu}_i(t) \triangleq \frac{1}{T_i(t)} \sum_{\substack{1  \le \tau \leq t, A_\tau=i}} D_\tau.
\end{align*} 


Next we recall
from 
Section \ref{sec:prob}
that the sampled distances 
are independent and
$\sigma$-sub-Gaussian (as points 
are sampled
with replacement).
Thus 
for all $i$, for all $\delta\in (0,1)$, 
with probability at least
$1-\delta$, $\mu_i \in [\hat{\mu}_i(t) - C_i(t), \hat{\mu}_i(t) + C_i (t)]$, 
where 
\begin{align} \label{eq:CI}
C_i(t) = \sqrt{\frac{2 \sigma^2 \log \frac{2}{\delta}}{T_i(t)}}.
\end{align}
When $T_i(t) \geq n$, we compute the exact average distance of point $i$ and set the confidence interval to zero.
Algorithm \ref{alg:UCB-Medoid}
describes the \texttt{Med-dit} algorithm. 

\begin{algorithm}[t]
   \caption{\texttt{Med-dit} \label{alg:UCB-Medoid}}
   \label{alg:cpca1}
\begin{algorithmic}[1]
   \STATE Evaluate 
   distances 
   of each point to a randomly 
   chosen point and build a $(1-\delta)$-confidence interval
   for the mean distance 
   of each point $i$:
   $[\hat{\mu}_i(1)-C_i(1), \ \hat{\mu}_i(1)+C_i(1)]$.    
  \WHILE \TRUE
   \STATE At iteration $t$, pick point $A_t$ that 
   minimises $\hat{\mu}_i(t-1)-C_i(t-1)$.    
    \IF{distances of point $A_t$ are 
   evaluated less than $n-1$ times}
    \STATE evaluate the distance
   of $A_t$ to a randomly picked point; update the confidence interval of $A_t$.
    \ELSE
    \STATE Set 
   $\hat{\mu}_i(t)$ to be the empirical mean 
   of distances of point $A_t$ by 
   computing its distance to all $(n-1)$ other
   points and set 
   $C_{A_t}(t)=0$.
  \ENDIF
   \IF {there exists a point $x_{i^*}$ 
   such that $\forall  i\neq i^*$, 
   $\hat{\mu}_{i^*}(t) + C_{i^*}(t) < \hat{\mu}_i(t) - C_i(t)$}
    \RETURN $x_{i^*}$.
  \ENDIF  
  \ENDWHILE
\end{algorithmic}
\end{algorithm}
\begin{theorem}\label{thrm:ub}
For $i\in[n]$, let $\Delta_i = \mu_i - \mu^*$. 
If we pick $\delta = \frac{2}{n^3}$ in  Algorithm \ref{alg:UCB-Medoid},
then 
with probability  $1-o(1)$, 
it returns the true medoid with the number of 
distance evaluations $M$ such that,
\begin{align}\label{eq:ub}
M \leq  \sum_{i \in [n]} \left( \frac{24 \sigma^2}{\Delta_i^2} \log n  \wedge 2n \right).
\end{align}
\end{theorem}

If $\frac{\Delta_i}{\sigma}$ 
is $\Theta(1)$, 
then
Theorem \ref{thrm:ub} gives that
\texttt{Med-dit} takes $O(n \log n)$ 
distance evaluations. 
In addition, if one assumed a Gaussian
prior 
on the average distance of points 
(as one would expect from Figure \ref{fig:distance_distribution}),
then \texttt{Med-dit} would also take $O(n \log n)$
distance evaluations in expectation over
the prior. See Appendix 
\ref{sec:Gau_prior}. 

If we use the same assumption as \texttt{TOP-RANK},
$i.e.$ the $\mu_i$ are i.i.d. Uniform $[c_0, c_1]$, then the
number of distance evaluations 
needed by \texttt{Med-dit} is 
$O(n^{\frac{3}{2}} \log^{\frac{1}{2}} n)$,
compared to $O(n^{\frac{5}{3}}\log ^\frac{4}{3} n)$ for \texttt{TOP-RANK}. 
Under these assumptions \texttt{RAND} takes $O(n^2)$
distance evaluations to return 
the true medoid.

We assume
that $\sigma$ is known in the proof below, 
whereas
we estimate them in the empirical evaluations. 
See Section \ref{sec:empi} for details.

\begin{remark}
With a small modification to \texttt{Med-dit}, the sub-Gaussian assumption in Section \ref{sec:prob} can be further relaxed to a finite variance assumption \citep{bubeck2013bandits}. 
One the other hand, one could theoretically adapt \texttt{Med-dit} to other best-arm algorithms
to find the medoid with $O(n\log\log n)$
distance computations at the sacrifice
of a large constant \citep{karnin2013almost}. Refer to Appendix \ref{sec:app_theory} for details.
\end{remark}

\begin{figure}[t]
	\centering
  \includegraphics[trim = {0 0 0 1.5cm}, clip, width=0.8\linewidth]{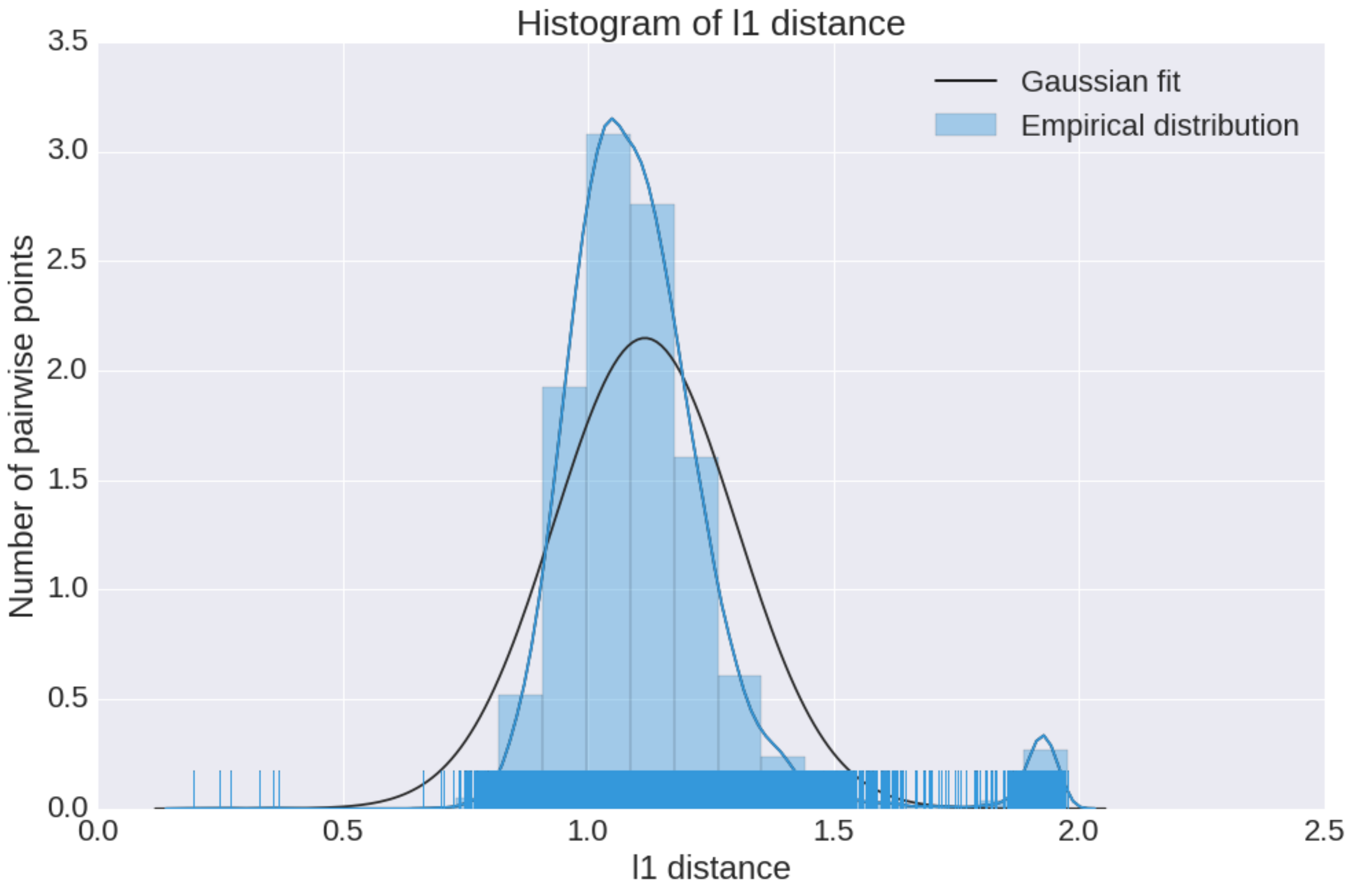}
\caption{ The distribution of mean distances of all the points in the
$20,000$ cell RNA-Seq dataset.\label{fig:distance_distribution}}
\end{figure}



\begin{proof}
Let $M$ be the total number of 
distance evaluations
when the algorithm 
stops. As defined above, let $T_i(t)$ be the number
of distance evaluations of point $i$ upto time $t$.

We first assume that 
$[\hat{\mu}_i(t)-C_i(t), \hat{\mu}_i(t)-C_i(t)]$
 are true 
$\left(1-\frac{2}{n^3}\right)$-confidence
interval (Recall that $\delta = \frac{2}{n^3}$) 
and show the result. Then we prove this statement.

Let $i^*$ be the true medoid.
We note 
that if we choose to update 
arm $i\neq i^*$ at time $t$, then we have 
\begin{align*}
\hat{\mu}_i(t) - C_i(t) \leq \hat{\mu}_{i^*}(t) - C_{i^*}(t).  
\end{align*}

For this to occur, at least one of the following three events must occur: 
\begin{align*}
& \mathcal{E}_1 = \left\{ \hat{\mu}_{i^*}(t) \geq \mu_{i^*}(t)+C_{i^*}(t) \right\},\\
& \mathcal{E}_2 = \left\{ \hat{\mu}_i(t) \leq \mu_i(t)-C_i(t) \right\}, \\
& \mathcal{E}_3 = \left\{ \Delta_i = \mu_i - \mu_{i^*} \leq 2 C_i(t) \right\}.
\end{align*}

To see this, note that if none of 
$\mathcal{E}_1, \mathcal{E}_2, \mathcal{E}_3$
occur,
we have 
\begin{align*}
\hat{\mu}_i(t) - C_i(t) \overset{(a)}{>} \mu_i - 2 C_i(t)\overset{(b)}{>}\mu_1 \overset{(c)}{>} \hat{\mu}_1 - C_1(t),
\end{align*}
where $(a)$, $(b)$, and $(c)$  follow because $\mathcal{E}_2$, 
$\mathcal{E}_3$, and $\mathcal{E}_1$ do not hold
respectively.

We note that as we compute 
$\left(1- \frac{2}{n^3}\right)-$confidence intervals 
at most $n$ times for each point. Thus we have 
at most $n^2$ computations of
$\left(1- \frac{2}{n^3}\right)-$confidence intervals
in total.

Thus $\mathcal{E}_1$ and $\mathcal{E}_2$
do not occur during any iteration 
with probability $1-\frac{2}{n}$,
because
\begin{align} \label{eq:sg_bound}
\text{w.p.}~(1-\frac{2}{n}): \vert \mu_i - \hat{\mu}_i(t) \vert \leq C_i(t),~\forall~i\in[n],~\forall~t.
\end{align}

This also implies that 
with
probability $1- \Theta\left(\frac{1}{n}\right)$
the 
algorithm does not stop unless 
the event $\mathcal{E}_3$,
a deterministic condition stops occurring. 

Let $\zeta_i$ be the 
iteration of the algorithm when it
evaluates a distance to point $i$ for 
the last time.
From the previous discussion, we 
have that the algorithm stops evaluating 
distances to
points $i$ when the following holds.
\begin{align*}
& C_i(\zeta_i) \le \frac{\Delta_i}{2}
\implies  \frac{\Delta_i}{2} \ge  \sqrt{\frac{2 \sigma^2 \log n^3}{T_i(\zeta_i)}}  \text{ or } C_i(\zeta_i) = 0, \\
&\implies T_i(\zeta_i)  \ge \frac{24 \sigma^2}{\Delta_i^2} \log n \text{ or }  T_i(\zeta_i) \ge 2n.
\end{align*}

Thus with probability $(1- o(1))$, the algorithm 
returns $i^*$ as the medoid with at most $M$ distance
evaluations, where
\begin{align*}
M \le \sum_{i \in [n]} T_i(\zeta_i)
\le \sum_{i \in [n]} \left( \frac{24 \sigma^2}{\Delta_i^2} \log n  \wedge 2n \right).
\end{align*}

To complete the proof, we next show that 
$[\hat{\mu}_i(t)-C_i(t), \hat{\mu}_i(t)-C_i(t)]$
 are true 
$\left(1-\frac{2}{n^3}\right)$-confidence
intervals of $\mu_i$.
We observe that if point $i$ is picked 
by \texttt{Med-dit} less than $n$ times at
time t, then $T_i(t)$ is equal to the number 
of times the point is picked. Further
$C_i(t)$ is the true $(1-\delta)$-confidence interval from 
Eq \eqref{eq:CI}.

However, if point $i$ is picked for the
$n$-th time at iteration $t$ (
line $7$ of Algorithm \ref{alg:UCB-Medoid})
then the empirical mean is computed
by evaluating all
$n-1$ distances (many distances again). 
Hence $T_i(t) = 2(n-1)$. As we know the 
mean distance of point $i$
exactly, $C_i(t) = 0$ is
still the true confidence interval. 

We remark that if a point $i$ is picked to be updated
the $(n+1)$-th time, as 
$C_i(t) =0$,  we have that  $\forall j \ne i$,
$$\hat{\mu}_i(t) + C_i(t) =\hat{\mu}_i(t) - C_i(t) < \hat{\mu}_j(t) - C_j(t).$$
This gives us that the stopping criterion is
satisfied and point $i$ is declared as the medoid.

\end{proof}

\vspace{-0.1 in}

\section{EMPIRICAL RESULTS}\label{sec:empi}
We empirically evaluate the performance
of \texttt{Med-dit} on two real world large-scale
high-dimensional  datasets:  
the
Netflix-prize dataset by \citep{netflixprize},
and 10x single cell RNA-Seq 
dataset \citep{10xdata}.

We picked these large datasets
because they 
have  sparse 
high-dimensional
feature vectors, 
and are at the throes of active efforts
to cluster using 
non-Euclidean distances. In 
both these datasets, we use
$1000$ randomly chosen points
to estimate the sub-Gaussianity 
parameter by cross-validation. 
We set $\delta = 1e$-$3$.

Running \texttt{trimed} and \texttt{TOPRANK} was computationally 
prohibitive due to the large size and high dimensionality ($\sim20,000$) of the above two datasets.
Therefore, we compare \texttt{Med-dit}  only to \texttt{RAND}.
We also ran \texttt{Med-dit} on three real world datasets  tested in \cite{NewFle} to compare it to \texttt{trimed} and \texttt{TOPRANK}.

\subsection{Single Cell RNA-Seq} 

Rapid advances in sequencing technology 
has enabled us to sequence 
DNA/RNA at a cell-by-cell basis and the single cell RNA-Seq
datasets can contain up to a million cells \citep{10xdata, 
indrop,
dropseq,10xold}. 

This technology involves 
sequencing an ensemble 
of single cells 
from a tissue, 
and obtaining the gene expressions
corresponding to each cell.
These
gene expressions are 
subsequently used to cluster
these cells in order 
to discover subclasses of cells
which would have been
hard to physically isolate
and study separately.
This technology 
therefore allows 
biologists to infer the diversity in a given tissue. 

We note that tens of thousands gene expressions
are measured in each cell, and
million of cells are sequenced in each
dataset.  Therefore
single cell RNA-Seq  has presented us 
with a very important large-scale
clustering problem with high dimensional data
\citep{10xold}. 
Moreover,
the distance metric of interest
here would be $\ell_1$ distance
as 
the features of each point 
are probability distribution,
Euclidean (or squared euclidean)
distances do not
capture the subtleties as
discussed in
\citep{batu2000testing, NtrKamZhaPacTse}. Further,
 very few genes are
expressed by most cells and the 
data is thus sparse.

\subsubsection{10xGenomics Mouse Dataset}\label{sec:10xdata}

We test the performance of \texttt{Med-dit} on
the single cell RNA-Seq dataset of \citep{10xdata}.
This dataset from 10xGenomics consists of 
$27,998$ genes-expression of $1.3$ million neurons cells 
from the cortex, hippocampus, and subventricular 
zone of a mice brain. These are clustered into 
$60$ clusters. Around $93\%$ of the gene expressions
in this dataset are zero.

We test \texttt{Med-dit} on 
subsets of  this dataset of two sizes:
\begin{itemize}
\item small dataset - $20,000$ cells (randomly chosen):  
We use this as we can compute the true medoid on this dataset
by brute force
and thus can compare the performance 
of \texttt{Med-dit} and \texttt{RAND}.
\item large dataset - $109,140$ cells (cluster 1) is the 
largest cluster in this dataset. We use the most commonly 
returned point as the true medoid for comparision. 
We note that this point is the same for both 
\texttt{Med-dit} and \texttt{RAND}, and has the smallest
distance among the top $100$ points returned in $1000$ trials
of both.
\end{itemize}

\begin{figure}[t]
\centering
  \includegraphics[width=.6\textwidth]{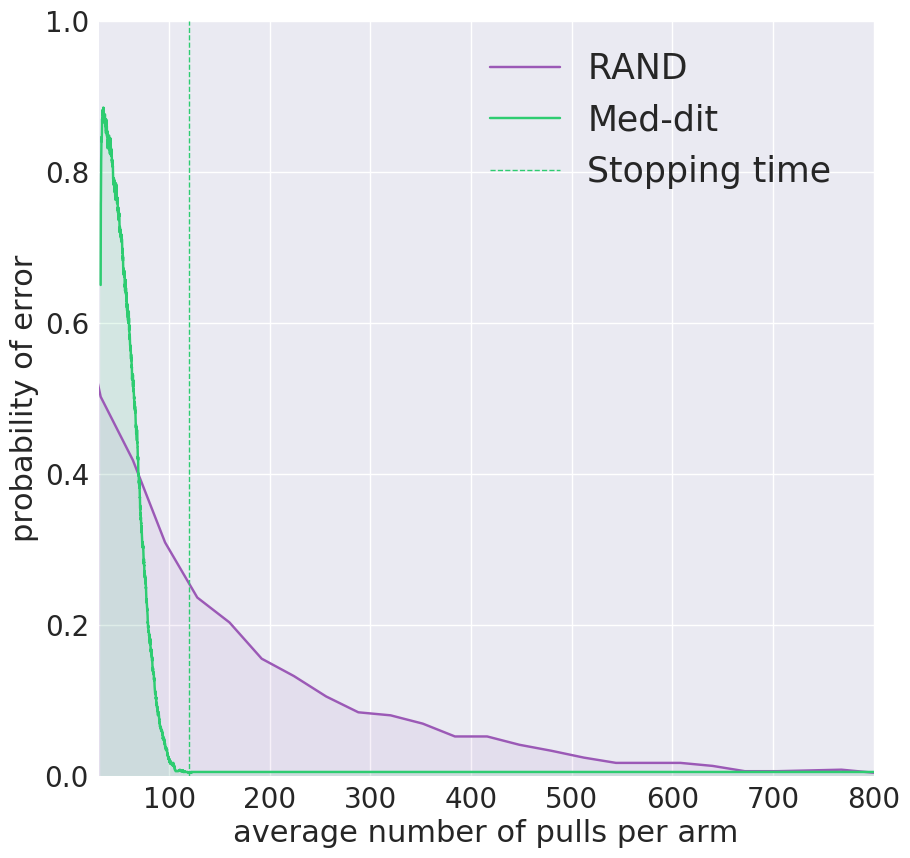}
\caption{We computed the true medoid of the data-set by brute force. The
$y$-axis shows  
 the probability that the point with the smallest estimated 
mean distance is \textit{not} the true medoid 
as a function of the number of pulls per arm.  
We note that  \texttt{Med-dit} has a stopping condition
while \texttt{RAND} does not. However we ignore the
stopping condition for \texttt{Med-dit} here.
\texttt{Med-dit} stops after $80$ distance 
evaluations per point without failing in any of the 
$1000$ trials, while 
\texttt{RAND} takes around $650$ distance evaluations to 
reach a $2\%$ probability of error.
}
\label{fig:scRNA_perf_20k}
\end{figure}

\textbf{Performance Evaluation : }
We compare the performance of \texttt{RAND} 
and \texttt{Med-dit} in 
Figures \ref{fig:performance_RNA_100k}
and \ref{fig:scRNA_perf_20k} on the large 
and the small dataset respectively.
We note that in both cases, 
 \texttt{RAND} needs 
$7$-$10$ times more distance 
evaluations to achieve $2\%$
error rate than what \texttt{Med-dit}
stops at (without an error in $1000$
trials). 

On the $20,000$ cell dataset, 
we note that \texttt{Med-dit}
stopped within $140$ distance evaluations
per arm in each of the $1000$ trials (with
$80$ distance evaluations per arm on average)
and never returned the wrong answer. 
\texttt{RAND} needs around $700$ distance
evaluations per point to obtain $2\%$ error rate.

On the $109,140$ cell dataset, 
we note that \texttt{Med-dit}
stopped within $140$ distance evaluations
per arm in each of the $1000$ trials (with
$120$ distance evaluations per arm on average)
and never returned the wrong answer. 
\texttt{RAND} needs around $700$ distance
evaluations per point to obtain $2\%$ error rate.

\subsection{Recommendation Systems}

With the explosion of e-commerce,
recommending products to users 
has been an important
avenue. 
Research into this took off with Netflix
releasing the Netflix-prize 
dataset \citep{netflixprize}.

Usually in the datasets involved here, 
we have
either rating given by users to different
products (movies, books, etc), or the
items bought by different users. The
task at hand is to recommend products
to a user based on behaviour of
similar users. Thus a common approach 
is to cluster similar users 
and to use trends 
in a cluster to recommend
products.

We note that in such datasets,
we typically have the behaviour
of millions of users and tens of thousands
of items. Thus recommendation
systems present us with an important
large-scale
clustering problem in high dimensions
\citep{DarMarWalGho}.
Further,
most users  buy (or rate) very few items on 
the inventory
and the 
data is thus sparse.
Moreover, the number of items bought
(or rated) by
users vary significantly and hence
distance
metrics 
that take this into account,
like  cosine distance
and  
Jaccard distance,
are of interest while clustering
as discussed in
\citep[Chapter~9]{LesRajUll}. 

\subsubsection{Netflix-prize Dataset}

\begin{figure}[t]
\centering
\minipage{0.48\textwidth}
  \includegraphics[width=1\linewidth]{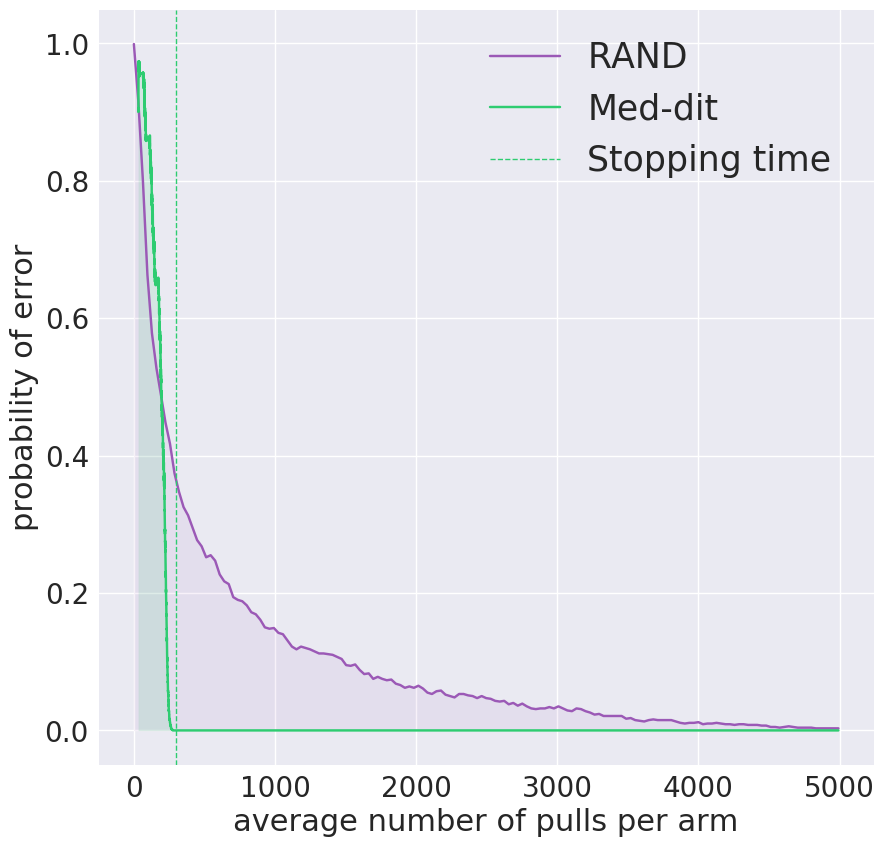}
  \endminipage
  \minipage{0.5\textwidth}
  \includegraphics[width=1\linewidth]{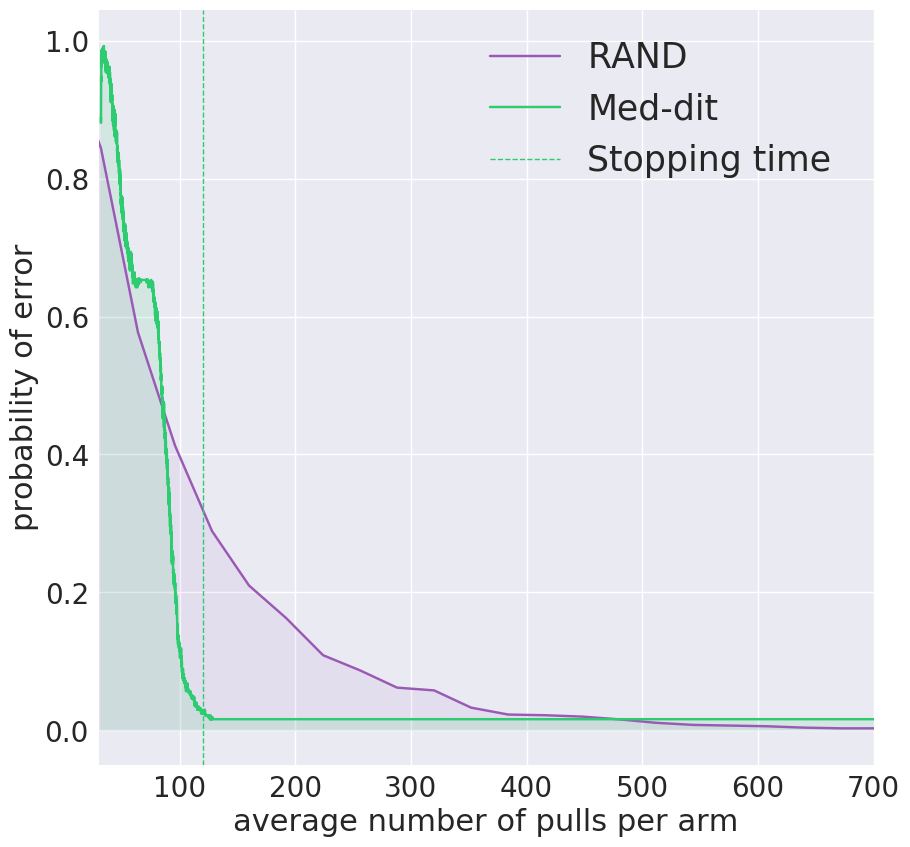}
  \endminipage
\caption{\textbf{Left}: the $20,000$ user Netflix-prize dataset. We computed the true medoid of the dataset by brute force. The
$y$-axis shows  
 the probability that the point with the smallest estimated 
mean distance is \textit{not} the true medoid 
as a function of the number of pulls per arm.  
We note that  \texttt{Med-dit} has a stopping condition
while \texttt{RAND} does not. However we ignore the
stopping condition for \texttt{Med-dit} here.
\texttt{Med-dit} stops after $500$ distance 
evaluations per point without failing in any of the 
$1000$ trials, while 
\texttt{RAND} takes around $4500$ distance evaluations to 
reach a $2\%$ probability of error.
\textbf{Right}: the $100, 000$ user Netflix-prize dataset. The $y$-axis shows  
the probability that the point with the smallest estimated 
mean distance is \textit{not} the true medoid 
as a function of the number of pulls per arm.  
We note that  \texttt{Med-dit} has a stopping condition
while \texttt{RAND} does not. However we ignore the
stopping condition for \texttt{Med-dit} here.
\texttt{Med-dit} stops after $500$ distance 
evaluations per point with $1\%$ probability of error, while 
\texttt{RAND} takes around $500$ distance evaluations to 
reach a $2\%$ probability of error. The computation of the true medoid here is  computationally prohibitive
and for comparision we use the most commonly returned point as the true medoid (which was same for both \texttt{Med-dit} and \texttt{RAND}) \label{fig:netflix_perf}}
\end{figure}

We test the performance of \texttt{Med-dit} on
the Netflix-prize dataset of \citep{netflixprize}.
This dataset from Netflix consists of 
ratings of 
$17,769$
movies by $480,000$
Netflix users.
Only $0.21\%$ of the 
entries in the matrix are non-zero.
As discussed in \citep[Chapter~9]{LesRajUll},
cosine distance is a popular 
metric considered here.

As in Section \ref{sec:10xdata}, we
use a small and a large subset 
of the dataset of size $20,000$ and
a $100,000$ respectively 
picked at random. For the former
we compute the true medoid by
brute force, while for the latter 
we use the most commonly 
returned point as the true medoid for comparision. 
We note that this point is the same for both 
\texttt{Med-dit} and \texttt{RAND}, and has the smallest
distance among the top $100$ points returned in $1000$ trials
of both.

\textbf{Performance Evaluation:}

On the $20,000$ user dataset (Figure \ref{fig:netflix_perf} left), 
we note that \texttt{Med-dit}
stopped within $600$ distance evaluations
per arm in each of the $1000$ trials (with
$500$ distance evaluations per arm on average)
and never returned the wrong answer. 
\texttt{RAND} needs around $4500$ distance
evaluations per point to obtain $2\%$ error rate.

On the $100, 000$ users on Netflix-Prize dataset (Figure \ref{fig:netflix_perf} right), we note that \texttt{Med-dit} stopped within 150 distance evaluations per arm in each of the 1000 trials (with 120 distance evaluations per arm on average) and returned the wrong answer only once. RAND needs around 500 distance evaluations per point to obtain $2\%$ error rate. 

\subsection{Performance on Previous Datasets}

We show the performance of 
\texttt{Med-dit} on data-sets ---  \texttt{Europe}  by \citep{UnbalanceSet} , 
 \texttt{Gnutella} by \citep{ripeanu2002mapping} ,  \texttt{MNIST} by \citep{lecun1998mnist}--- 
used in prior works in Table \ref{tab:perf}.
The number of distance evaluations
for \texttt{TOPRANK} and \texttt{trimed} 
are taken from \citep{NewFle}.
We note that the performance of \texttt{Med-dit}
on real world datasets which have low-dimensional feature vectors (like Europe)
is
worse than that of \texttt{trimed}, while it is better on
 datasets (like MNIST) with high-dimensional feature vectors. 
 
 We conjecture
 that this is due to the fact the distances in high 
 dimensions are average of simple functions (over a each dimension) and hence have Gaussian-like behaviour due to the Central Limit Theorem (and the delta method \citep{van1998asymptotic}). 
 
\begin{table}[t]
\centering
\begin{tabular}{||c|c|c|c|c||}
\hline
Dataset & $n, d$ &\texttt{TOPRANK}  & \texttt{trimed} & \texttt{med-dit}\\
\hline
Europe&$160k, 2$& $176k$  & $2862$ & $3514$\\
\hline
Gnutella &$6.3k$, NA& $7043$  & $6328$ & $83$\\
\hline
MNIST &$6.7k, 784$& $7472$  & $6514$& $91$\\
\hline
\end{tabular}
\caption{This table shows the performance on real-world 
data-sets picked from \cite{NewFle}. We note that performance
of \texttt{Med-dit} on  
 datasets, where the features
are in low dimensions is not as 
good as that of \texttt{trimed}, in high
dimensions \texttt{Med-dit} 
performs much better}.
\label{tab:perf}
\end{table}%

\subsection{Inner workings of \texttt{Med-dit}}

\textbf{Empirical running time:}
We computed the true mean distance $\mu_i$, and 
variance $\sigma$, for all the points in the small dataset 
by brute force.  Figure \ref{fig:distance_distribution} shows 
the histrogram of $\mu_i$'s. Using $\delta=1e\text{-}3$ in
Theorem \ref{thrm:ub}, the theoretical average number 
of distance evaluations is $266$. 
Empirically over $1000$ experiments, the average 
number of distance computations is $80$. This suggests 
that the theoretic analysis captures the 
essence of the algorithm.

\textbf{Progress of \texttt{Med-dit} : }
We see that at any iteration of 
 \texttt{Med-dit}, 
 only the points whose lower confidence
interval are smaller than the smallest upper confidence
interval among all points 
have their distances evaluated. At any iteration,
we say that such points are \textit{under 
consideration}. We note that a point 
that goes out of consideration can be
under consideration at a later iteration (this
happens if the smallest upper confidence
interval increases). We illustrate
the fraction of points under consideration 
at different iterations (indexed by the number
of distance evaluations per point)
in the left panel of Figure \ref{fig:progress_100k}. 
We pick $10$ snapshots of the algorithm
and illustrate the true distance distributions
of all points \textit{under consideration}
at that epoch
in the right panel of Figure 
\ref{fig:progress_100k}.
We note that both the 
mean and the 
standard deviation
of these distributions
decrease as the algorithm 
progresses.\footnote{ When the number of 
points \textit{under consideration} 
is more than $2000$, we 
draw the distribution using 
the true distances of $200$ 
random samples.}

\textbf{Confidence intervals at stopping time : }
\texttt{Med-dit} seems to compute 
the mean distance of a few points
exactly. These are usually points
with the mean distance close to 
the that of the medoid.  
The points whose mean distances 
are farther away from that 
of the medoid 
have fewer distances  evaluated
to have the confidence intervals 
rule them out as the medoid.
The $99.99\%$ confidence
interval of the top $150$ points
 in a run of the experiment
 are shown in Figure \ref{fig:CI_20k}
 for the $20,000$ cell RNA-Seq dataset.
 This allows the algorithm to save
 on distance computations 
 while returning the correct 
 answer.

\textbf{Speed-accuracy tradeoff}
Setting the value of $\delta$ between $0.01$ to $0.1$ will result in smaller confidence intervals 
around the estimates $\hat{\mu}_i$, which will ergo reduce the number of distance evaluations.
One could also run \texttt{Med-dit} for a fixed number of iterations and return the point with the smallest mean estimate. Both of these methods improve the running time at the cost of accuracy. Practically, the first method has a better trade-off between accuracy and running time whereas the second method has a deterministic stopping time.

\begin{figure*}
\minipage{0.48\textwidth}
  \includegraphics[width=1\linewidth]{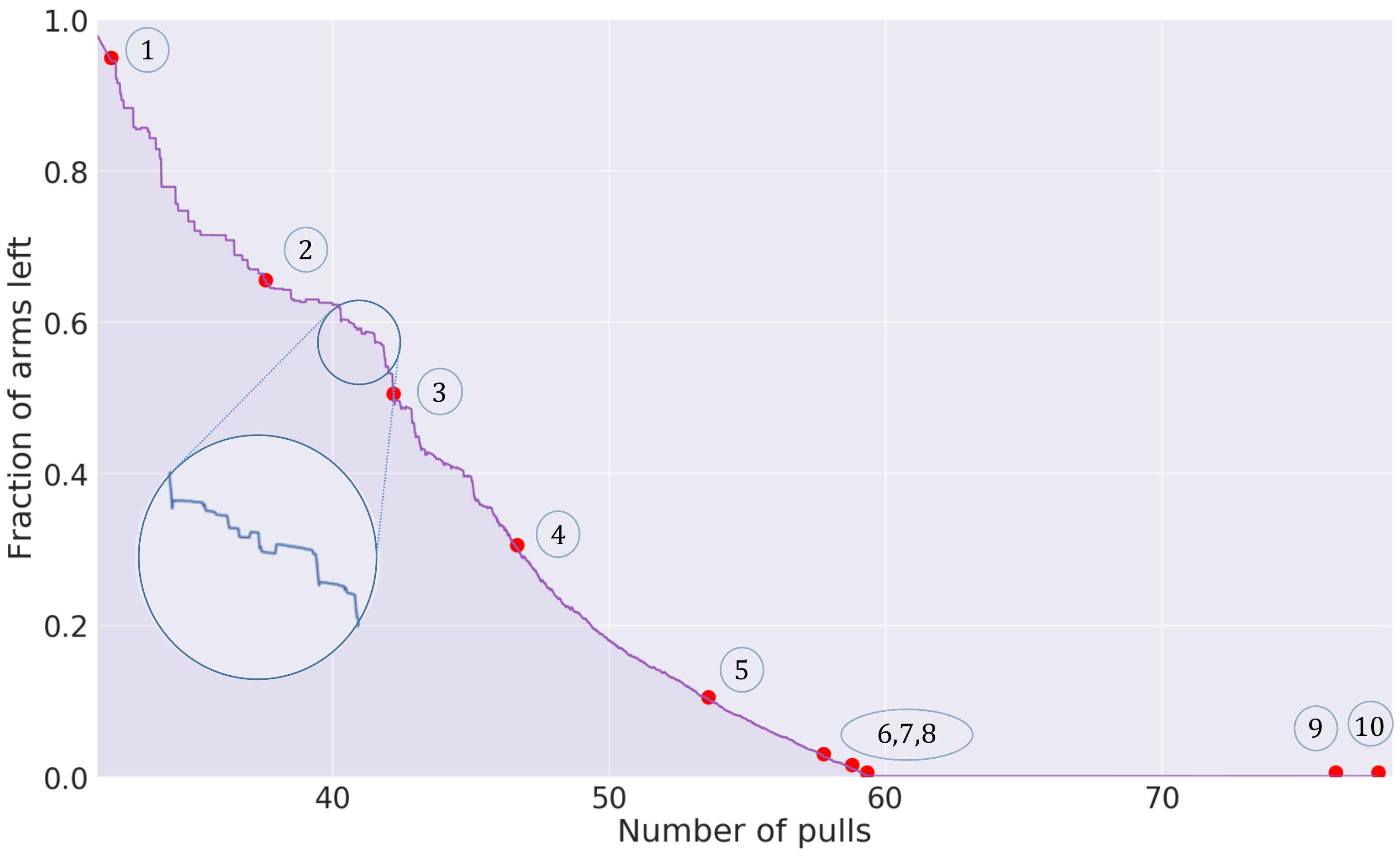}
  \endminipage
  \minipage{0.52\textwidth}
  \includegraphics[width=1\linewidth]{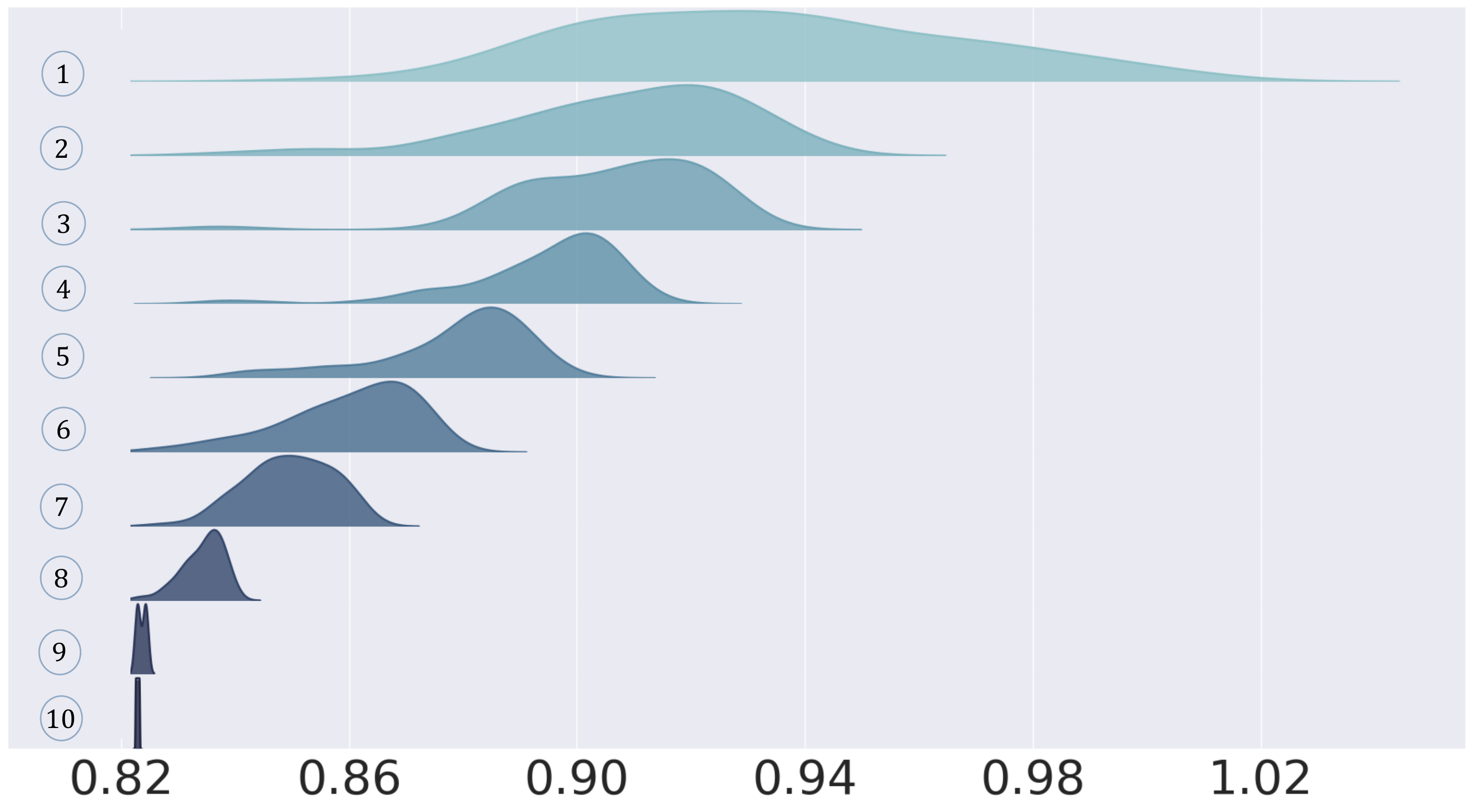}
  \endminipage
\caption{ \label{fig:progress_100k}
In the left panel, we show the number of points
\textit{under consideration} at different 
iterations in \texttt{Med-dit}.
We note that though the number 
of points \textit{under consideration}
show a decreasing trend, they do not decrease
monotonically. In the right panel,
 we show the 
  distribution of the true distances among the
  arms that are \textit{under consideration}
  at various snapshots of the algorithm. We note that
  both
  the mean and the variance
  of the distributions keeps decreasing. 
The last snapshot shows only the
distance of the point declared 
a medoid.
 }
\end{figure*}

\begin{figure}
\centering
 \includegraphics[clip,width=0.7\linewidth]{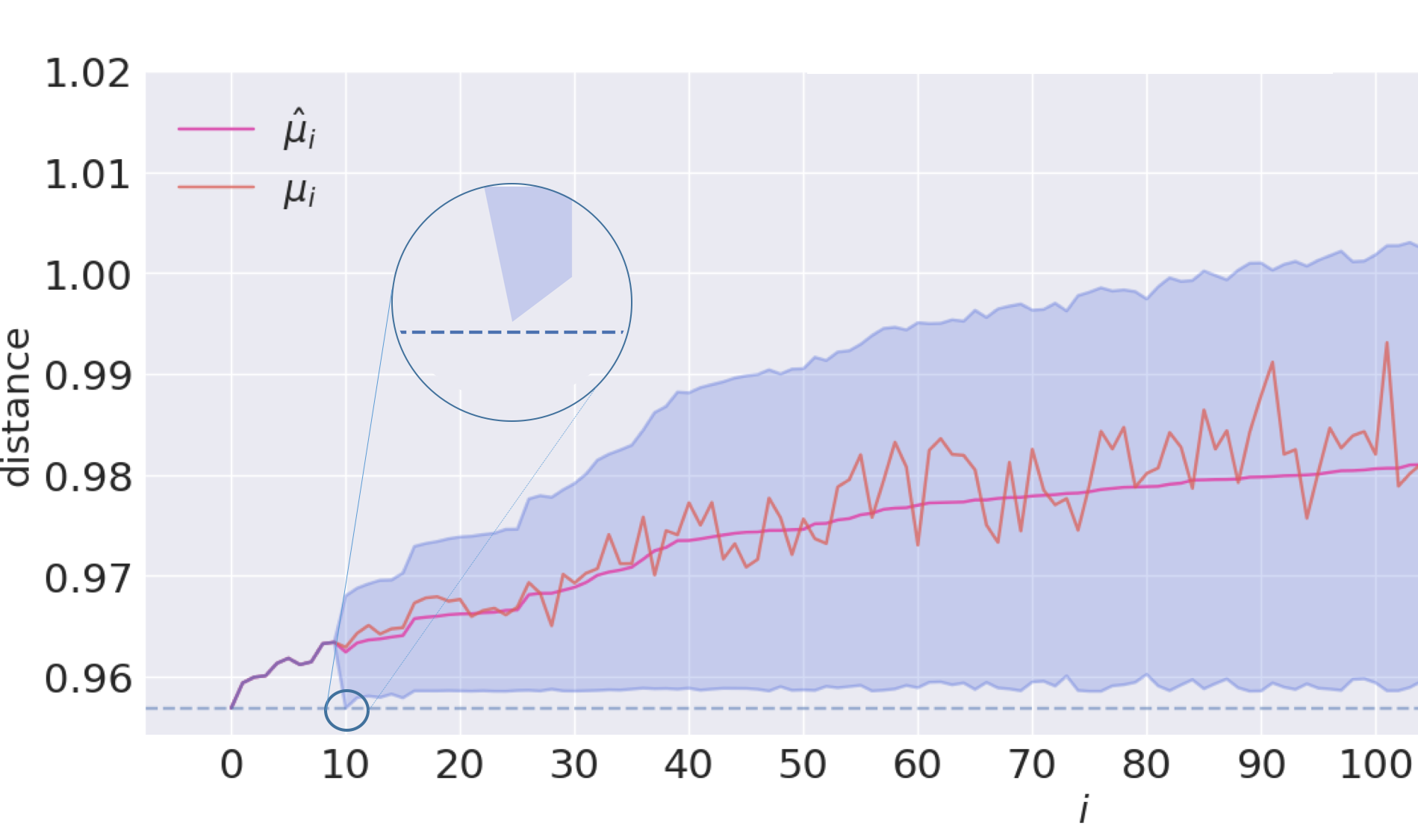}
  \caption{\label{fig:CI_20k} $99.99\%$-Confidence Intervals 
  of the $150$ points with smallest estimated mean distance in the $20,000$ 
  cell dataset at the termination of \texttt{Med-dit}. 
  The true means are shown by the orange line, while the
  estimated means are shown by the pink line. 
  We note that the distances of $9$ points 
  are computed to all other
  points. 
  We also note that mean of the medoid 
  is less than the $99.99\%$-lower
  confidence bound of all the other
  points.  
  }
\end{figure}

\let\OLDthebibliography\thebibliography
\renewcommand\thebibliography[1]{
  \OLDthebibliography{#1}
  \setlength{\parskip}{0pt}
  \setlength{\itemsep}{0pt plus 0.29ex}
}

\clearpage
\bibliographystyle{plain}
\bibliography{ref}
\clearpage

\newpage 
\clearpage
\onecolumn
\begin{center}
\textbf{\large Appendices}
\end{center}
\setcounter{section}{0}

\section{ The $O(n \log n)$ Distance Evaluations Under Gaussian Prior \label{sec:Gau_prior}}


We assume that the mean distances
of each point $\mu_i$
are i.i.d. samples of 
${N} (\gamma, 1)$. 
We note that this implies that
$\Delta_i, \ 1\le i \le n$ are 
$n$ iid random variables.
Let $\Delta$ be a random
variable with the same law
as $\Delta_i$.

From the concentration 
of the minimum of $n$ gaussians,
we have that 
$$\min_i \mu_i + \sqrt{2 \log n} \overset{p}{\rightarrow} \gamma. $$

This gives us that 
$$\Delta  - \sqrt{2 \log n} \overset{d}{\rightarrow}  \mathcal{N}( 0,1).$$

We note that by Eq \eqref{eq:ub}, we have that the
expected number of distance evaluations $M$ is of
the order of 
\begin{align*}
\mathbb{E}[M] &\le n\mathbb{E} \left[\frac{\log n}{\Delta^2} \wedge n \right],
\end{align*}
where the expection is taken with respect to the randomness of $\Delta$.

To show this, it is enough to show that, 
\begin{align*}
\mathbb{E} \left[\frac{\log n}{\Delta^2} \wedge n \right] \le  C \log n,
\end{align*}
for some constant $C$.

To compute that for this prior, we divide the real line into 
three intervals, namely
\begin{itemize}
\item $\left(-\infty, \sqrt{\frac{{\log n}}{n}} \right]$,
\item $\left(\sqrt{\frac{{\log n}}{n}}, c\sqrt{\log n} \right)$,
\item $\left[ c\sqrt{\log n}  ,\infty\right)$,
\end{itemize}
 and compute the expectation on these three 
ranges. We note that for if $\Delta \in 
(-\infty, \sqrt{\frac{{\log n}}{n}} ]$, 
while for $\Delta \in 
(\sqrt{\frac{{\log n}}{n}}, \infty ]$, 
$ \frac{\log n}{\Delta^2} \le n$. Thus we 
have that,
\begin{align*}
\mathbb{E} \left[\frac{\log n}{\Delta^2} \wedge n \right] &\le \overbrace{\mathbb{E} \left[n
\mathbb{I}\left(\Delta \leq \sqrt{\frac{{\log n}}{n}} \right)\right]}^{\mathbf{I}} + 
\overbrace{ \lim_{\delta,\epsilon \rightarrow 0}  \mathbb{E} \left[\frac{\log n}{\Delta^2} 
\mathbb{I}\left(\sqrt{\frac{\log n}{n^{(1-\epsilon)}}}  \le \Delta \le 
c\frac{\sqrt{\log n}}{n^{\delta}}  \right)\right]  }^{\mathbf{II}}  \\
&  \quad \quad  + \underbrace{\mathbb{E} \left[\frac{\log n}{\Delta^2}
\mathbb{I}\left( \Delta \ge 
c\sqrt{\log n} \right)\right]}_{\mathbf{III}},
\end{align*}
where we use the Bounded Convergence Theorem
to establish $\mathbf{II}$.

We next show that all three terms in the above equation
are of $O(\log n)$. We first show the easy cases of 
$\mathbf{I}$ and $\mathbf{III}$ and then proceed to
$\mathbf{III}$.

\begin{itemize}
\item To establish that $\mathbf{I}$ is $O(\log n)$, we start by defining,
$q_1 = P[\Delta < \frac{\sqrt{\log n}}{n}]$
\begin{align*}
\mathbb{E} \left[n
\mathbb{I}\left(\Delta \leq \frac{\sqrt{\log n}}{n} \right)\right]  &= nq_1.
\end{align*}
Further note that,
\begin{align*}
q_1 &\le \exp \left( - \frac{1}{2} \left( \sqrt{2\log n} - \sqrt{\frac{\log n }{n}} \right)^{2} \right),\\
&= \left(\frac{1}{n} \right)^{(1-\frac{1}{ \sqrt{2n}})^2}.
\end{align*}

Thus 
\begin{align*}
\frac{q_1 n}{ \log n} &\le \frac{n^{1 - (1-\frac{1}{ \sqrt{2n}})^2}}{\log n},\\
&\le \exp( (\sqrt{\frac{2}{n}}(1+o(1)) \log n - \log \log n),\\
&= o(1).
\end{align*}

\item To establish that $\mathbf{III}$ is $O(\log n)$, we note that,
\begin{align*}
\mathbb{E} \left[\frac{\log n}{\Delta^2} 
\mathbb{I}\left( \Delta \ge 
c\sqrt{\log n} \right)\right] &\le \frac{1}{c^2}  P(\Delta \ge c \sqrt{\log n}),\\
&\le  \frac{1}{c^2},\\
&= \Theta(1).
\end{align*}

\item Finally to establish that $\mathbf{II}$ is $O(\log n)$, we note that,
\begin{align*}
 \lim_{\delta,\epsilon \rightarrow 0} \mathbb{E} \left[\frac{\log n}{\Delta^2} 
\mathbb{I}\left(\sqrt{\frac{\log n}{n^{(1-\epsilon)}}}  \le \Delta \le 
c\frac{\sqrt{\log n}}{n^{\delta}}  \right)\right]  &\le  \lim_{\delta,\epsilon \rightarrow 0} n^{2-\epsilon} P\left(\Delta \le 
c\frac{\sqrt{\log n}}{n^{\delta}}\right),\\
&\le \lim_{\delta,\epsilon \rightarrow 0}  n^{1-\epsilon} \exp \left( - \frac{1}{2} \left( \sqrt{2\log n} - \sqrt{\log n}\frac{c}{n^{\delta}} \right)^{2} \right),\\
&= \lim_{\delta,\epsilon \rightarrow 0} n^{1-\epsilon} \left( \frac{1}{n} \right) ^{(1 - \frac{c}{\sqrt{2} n^{\delta}})^2},\\
& = \lim_{\delta,\epsilon \rightarrow 0}  n^{1 - \epsilon - (1 - \frac{c}{\sqrt{2} n^{\delta}})^2 },\\
&= \lim_{\delta,\epsilon \rightarrow 0} n^{- \epsilon + \frac{\sqrt{2}c}{n^{\delta}} - \frac{c^2}{2 n^{2 \delta}}}.
\end{align*}

Letting $\delta$ to go to $0$ faster than $\epsilon$, we see that,
\begin{align*}
 \lim_{\delta,\epsilon \rightarrow 0} \mathbb{E} \left[\frac{\log n}{\Delta^2} 
\mathbb{I}\left(\sqrt{\frac{\log n}{n^{(1-\epsilon)}}}  \le \Delta \le 
c\frac{\sqrt{\log n}}{n^{\delta}}  \right)\right]  &\le O(\log n).
\end{align*}

\end{itemize}

This gives us that under this model, we have 
that under this model,
\begin{align*}
\mathbb{E}[M] \le O(n \log n).
\end{align*}

\clearpage

\section{Extensions to the Theoretical Results}\label{sec:app_theory}
In this section we discuss two extensions to the theoretical results presented in the main article: 
\begin{enumerate}
\item  relax  the sub-Gaussian assumption
used
for the analysis of \texttt{Med-dit}.  
We see that assuming that the distances random 
variables have finite variances 
is enough for a variant of \texttt{Med-dit} to need 
essentially the same number of distance
evaluations as Theorem \ref{thrm:ub}. 
The variant would use
a different estimator (called Catoni's M-estimator \citep{catoni2012challenging}) for the 
mean distance of each point instead of 
empirical average leveraging the work of \cite{bubeck2013bandits}.  
This is based on 
 the work of \cite{bubeck2013bandits}.
\item note that there exists an algorithm which 
can compute the medoid
with $O(n \log \log n)$ distance evaluations. 
The algorithm called \texttt{exponential-gap}
algorithm \cite{karnin2013almost} and is discussed
below.
\end{enumerate}

%
 
\subsection{Weakening the Sub-Gaussian Assumption}
We note  that in order to have the $O(n\log n)$ sample complexity, \texttt{Med-dit} relies on a concentration bound where the tail probability decays exponentially. 
This is achieved by assuming that for each point $x_i$, the random variable of sampling with replacement from $\mathcal{D}_i$ is $\sigma$-sub-Gaussian in Theorem \ref{thrm:ub}.
As a result, we can have the sub-Gaussian tail bound that for any point $x_i$ at time $t$, with probability at least $1-\delta$, the empirical mean $\hat{\mu}_i$ satisfies
\begin{align*}
\vert \mu_i - \hat{\mu}_i\vert \leq \sqrt{\frac{2\sigma^2 \log \frac{2}{\delta}}{T_i(t)}} .
\end{align*} 
In fact, as pointed out by \cite{bubeck2013bandits}, to achieve the $O(n\log n)$ sample complexity, all we need is a performance guarantee like the one shown above for the empirical mean. 
To be more precise, we need the following property: 
\begin{assumption} \citep{bubeck2013bandits} \label{asmp:tail_bd} Let $\epsilon \in (0,1]$ be a positive parameter and let $c,v$ be positive constants.
Let $X_1,\cdots,X_T$ be i.i.d. random variables with finite mean $\mu$.
Suppose that for all $\delta \in (0,1)$, there exists an estimator $\hat{\mu}=\hat{\mu}(T,\delta)$ such that, with probability at least $1-\delta$,
\begin{align*}
\vert \mu - \hat{\mu} \vert \leq v^{\frac{1}{1+\epsilon}} \left( \frac{c \log \frac{2}{\delta}}{T}\right) ^{\frac{\epsilon}{1+\epsilon}}.
\end{align*}
\end{assumption}
\begin{remark}
If the distribution of $X_j$ satisfies $\sigma$-sub-Gaussian condition, then Assumption \ref{asmp:tail_bd} is satisfied for $\epsilon=1$, $c=2$, and variance factor $v=\sigma^2$.
\end{remark}

However, Assumption \ref{asmp:tail_bd} can be satisfied with conditions much weaker than the sub-Gaussian condition. 
One way is by substituing the empirical mean estimator by some refined mean estimator that gives the exponential tail bound. 
Specifically, as suggested by \cite{bubeck2013bandits}, we can use Catoni's M-estimator \citep{catoni2012challenging}.

Catoni's M-estimator is defined as follows: let $\psi: \mathbb{R} \to \mathbb{R}$ be a continuous strictly increasing function satisfying
\begin{align*}
-\log (1-x+\frac{x^2}{2}) \leq \psi (x) \leq \log (1+x+\frac{x^2}{2}).
\end{align*}
Let $\delta \in (0,1)$ be such that $T > 2 \log (\frac{1}{\delta})$ and introduce 
\begin{align*}
\alpha_\delta = \sqrt{\frac{2 \log \frac{1}{\delta}}{T(\sigma^2 + \frac{2 \sigma^2 \log \frac{1}{\delta}}{T-2\log \frac{1}{\delta}})}}.
\end{align*}
If $X_1,\cdots,X_T$ are i.i.d. random variables, the Catoni's estimator is defined as the unique value $\hat{\mu}_C = \hat{\mu}_C (T,\delta)$ such that 
\begin{align*}
\sum_{i=1}^n \psi(\alpha_\delta (X_i - \hat{\mu}_C)) =0.
\end{align*}
Catoni \citep{catoni2012challenging} proved that if $T \geq 4 \log \frac{1}{\delta}$ and the $X_j$ have mean $\mu$ and variance at most $\sigma^2$, then with probability at least $1-\delta$, 
\begin{align}\label{eq:ci_catoni}
\vert \hat{\mu}_C - \mu \vert \leq 2 \sqrt{\frac{\sigma^2 \log \frac{2}{\delta}}{T}}.
\end{align}
The corresponding modification to \texttt{Med-dit} is as follows.
\begin{enumerate} 
\item For the initialization step, sample each point $4 \log \frac{1}{\delta}$ times to meet the condition for the concentration bound of the Catoni's M-estimator. 
\item For each arm $i$, if $T_i(t) < n$, maintain the $1-\delta$ confidence interval $[\hat{\mu}_{C,i}-C_i(t), \hat{\mu}_{C,i}+C_i(t)]$, where  $\hat{\mu}_{C,i}$ is the Catoni's estimator of $\mu_i$, and 
\begin{align*}
C_i(t) = 2 \sqrt{\frac{\sigma^2 \log \frac{2}{\delta}}{T_i(t)}}.
\end{align*}
\end{enumerate}

\begin{proposition}
For $i\in[n]$, let $\Delta_i = \mu_i - \mu^*$. 
If we pick $\delta = \frac{1}{n^3}$ in the above algorithm,
then with probability  $1-o(1)$, 
it returns the true medoid with the with number of 
distance evaluations $M$ such that,
\begin{align*}
M \leq  12 n \log n + \sum_{i \in [n]} \left( \frac{48 \sigma^2}{\Delta_i^2} \log n  \wedge 2n \right)
\end{align*}
\end{proposition}

\begin{proof}
Let $\delta = \frac{2}{n^3}$.
The initialization step takes an extra $12 n \log n$ distance computations. 
Following the same proof as Theorem \ref{thrm:ub}, we can show that the modified algorithm returns the true medoid with probability at least $1-\Theta(\frac{1}{n})$, and apart from the initialization, the total number of distance computations can be upper bounded by 
\begin{align*}
\sum_{i \in [n]} \left( \frac{48 \sigma^2}{\Delta_i^2} \log n  \wedge 2n \right).
\end{align*}
So the total number of distance computations can be upper bounded by 
\begin{align*}
M \leq  12 n \log n + \sum_{i \in [n]} \left( \frac{48 \sigma^2}{\Delta_i^2} \log n  \wedge 2n \right).
\end{align*}
\end{proof}

\begin{remark}
By using the Catoni's estimator, instead of 
the empirical mean, we need a much weaker
assumption that the distance evaluations have finite variance
to get the same results as Theorem \ref{thrm:ub}. 
\end{remark}

\subsection{On the $O(n\log \log n)$ Algorithm}
The best-arm algrithm \texttt{exponential-gap} \citep{karnin2013almost} can be directly applied on the medoid problem, which takes $O(\sum_{i \neq i^*} \Delta_i^{-2} \log \log \Delta_i^{-2})$ distance evaluations, essentially $O(n\log \log n)$ if $\Delta_i$ are constants. 
It is an variation of the familiy of action elimination algorithm for the best-arm problem. 
A typical action elimination algorithm proceeds as follows: Maintaining a set $\Omega_k$ for $k=1,2,\cdots$, initialized as $\Omega_1=[n]$.
Then it proceeds in epoches by sampling the arms in $\Omega_k$ a predetermined number of times $r_k$, and maintains arms according to the rule:
\begin{align*}
\Omega_{k+1} = \{ i \in \Omega_k: \hat{\mu}_{a} + C_{a}(t) < \hat{\mu}_{i} - C_{i}(t)\},
\end{align*}
where $a \in \Omega_k$ is a reference arm, e.g. the arm with the smallest  $\hat{\mu}_{i} + C_{i}(t)$. 
Then the algorithm terminates when $\Omega_k$ contains only one element. 

The above vanilla version of the action elimination algorithm takes $O(n \log n)$ distance evaluations, same as \texttt{Med-dit}. 
The improvement by \texttt{exponential-gap} is by observing that the suboptimal $\log n$ factor is due to the large deviations of $\vert \hat{\mu}_{a}-\mu_a\vert $ with $a=\arg\min_{i\in\Omega_k} \hat{\mu}_i$. 
Instead, \texttt{exponential-gap} use a subroutine \texttt{median elimination} \citep{even2006action} to determine an alternative reference arm $a$ with smaller deviations and allows for the removal of the $\log n$ term,  
where \texttt{median elimination} takes $O(\frac{n}{\epsilon^2} \log \frac{1}{\delta})$ distance evaluations to return a $\epsilon$-optimal arm. 
However, this will introduce a prohibitively large constant due to the use of \texttt{median elimination}.
Regarding the technical details, we note both paper \citep{karnin2013almost,even2006action} assume the boundedness of the random variables for their proof, which is only used to have the hoefflding concentration bound. 
Therefore, with our sub-Gaussian assumption, the proof will follow symbol by symbol, line by line.

\end{document}